\documentclass[10pt,twocolumn,letterpaper]{article}

\usepackage[pagenumbers]{cvpr} 

\usepackage{graphicx}
\usepackage{amsmath}
\usepackage{amssymb}
\usepackage{booktabs}
\usepackage{algorithm,algorithmic}
\usepackage{wrapfig}
\usepackage{caption}
\usepackage{subcaption}
\usepackage{sidecap}

\usepackage{bm}

\usepackage{amsthm}

\usepackage[on]{editing}

\usepackage{xcolor}

\usepackage[pagebackref,breaklinks,colorlinks]{hyperref}

\usepackage[capitalize]{cleveref}
\crefname{section}{Sec.}{Secs.}
\Crefname{section}{Section}{Sections}
\Crefname{table}{Table}{Tables}
\crefname{table}{Tab.}{Tabs.}

\def\cvprPaperID{2916} 
\def\confName{CVPR}
\def\confYear{2022}

\begin{document}

\def\Blue{\color{blue}}
\def\Purple{\color{purple}}

\def\A{{\bf A}}
\def\a{{\bf a}}
\def\B{{\bf B}}
\def\b{{\bf b}}
\def\C{{\bf C}}
\def\c{{\bf c}}
\def\D{{\bf D}}
\def\d{{\bf d}}
\def\E{{\bf E}}
\def\e{{\bf e}}
\def\f{{\bf f}}
\def\F{{\bf F}}
\def\K{{\bf K}}
\def\k{{\bf k}}
\def\L{{\bf L}}
\def\H{{\bf H}}
\def\h{{\bf h}}
\def\G{{\bf G}}
\def\g{{\bf g}}
\def\I{{\bf I}}
\def\R{{\bf R}}
\def\X{{\bf X}}
\def\Y{{\bf Y}}
\def\OO{{\bf O}}
\def\oo{{\bf o}}
\def\P{{\bf P}}
\def\Q{{\bf Q}}
\def\r{{\bf r}}
\def\s{{\bf s}}
\def\S{{\bf S}}
\def\t{{\bf t}}
\def\T{{\bf T}}
\def\x{{\bf x}}
\def\y{{\bf y}}
\def\z{{\bf z}}
\def\Z{{\bf Z}}
\def\M{{\bf M}}
\def\m{{\bf m}}
\def\n{{\bf n}}
\def\U{{\bf U}}
\def\u{{\bf u}}
\def\V{{\bf V}}
\def\v{{\bf v}}
\def\W{{\bf W}}
\def\w{{\bf w}}
\def\0{{\bf 0}}
\def\1{{\bf 1}}
\def\N{{\bf N}}

\def\AM{{\mathcal A}}
\def\EM{{\mathcal E}}
\def\FM{{\mathcal F}}
\def\TM{{\mathcal T}}
\def\UM{{\mathcal U}}
\def\XM{{\mathcal X}}
\def\YM{{\mathcal Y}}
\def\NM{{\mathcal N}}
\def\OM{{\mathcal O}}
\def\IM{{\mathcal I}}
\def\GM{{\mathcal G}}
\def\PM{{\mathcal P}}
\def\LM{{\mathcal L}}
\def\MM{{\mathcal M}}
\def\DM{{\mathcal D}}
\def\SM{{\mathcal S}}
\def\RB{{\mathbb R}}
\def\EB{{\mathbb E}}

\def\tx{\tilde{\bf x}}
\def\ty{\tilde{\bf y}}
\def\tz{\tilde{\bf z}}
\def\hd{\hat{d}}
\def\HD{\hat{\bf D}}
\def\hx{\hat{\bf x}}
\def\hR{\hat{R}}

\def\Ome{\mbox{\boldmath$\omega$\unboldmath}}
\def\bet{\mbox{\boldmath$\beta$\unboldmath}}
\def\et{\mbox{\boldmath$\eta$\unboldmath}}
\def\ep{\mbox{\boldmath$\epsilon$\unboldmath}}
\def\ph{\mbox{\boldmath$\phi$\unboldmath}}
\def\Pii{\mbox{\boldmath$\Pi$\unboldmath}}
\def\pii{\mbox{\boldmath$\pi$\unboldmath}}
\def\Ph{\mbox{\boldmath$\Phi$\unboldmath}}
\def\Ps{\mbox{\boldmath$\Psi$\unboldmath}}
\def\pss{\mbox{\boldmath$\psi$\unboldmath}}
\def\tha{\mbox{\boldmath$\theta$\unboldmath}}
\def\Tha{\mbox{\boldmath$\Theta$\unboldmath}}
\def\muu{\mbox{\boldmath$\mu$\unboldmath}}
\def\Si{\mbox{\boldmath$\Sigma$\unboldmath}}
\def\Gam{\mbox{\boldmath$\Gamma$\unboldmath}}
\def\gamm{\mbox{\boldmath$\gamma$\unboldmath}}
\def\Lam{\mbox{\boldmath$\Lambda$\unboldmath}}
\def\De{\mbox{\boldmath$\Delta$\unboldmath}}
\def\vps{\mbox{\boldmath$\varepsilon$\unboldmath}}
\def\Up{\mbox{\boldmath$\Upsilon$\unboldmath}}
\def\Lap{\mbox{\boldmath$\LM$\unboldmath}}
\newcommand{\ti}[1]{\tilde{#1}}

\def\tr{\mathrm{tr}}
\def\etr{\mathrm{etr}}
\def\etal{{\em et al.\/}\,}
\newcommand{\indep}{{\;\bot\!\!\!\!\!\!\bot\;}}
\def\argmax{\mathop{\rm argmax}}
\def\argmin{\mathop{\rm argmin}}
\def\vec{\text{vec}}
\def\cov{\text{cov}}
\def\dg{\text{diag}}

\newcommand{\tabref}[1]{Table~\ref{#1}}
\newcommand{\secref}[1]{Sec.~\ref{#1}}
\newcommand{\figref}[1]{Fig.~\ref{#1}}
\newcommand{\lemref}[1]{Lemma~\ref{#1}}
\newcommand{\thmref}[1]{Theorem~\ref{#1}}
\newcommand{\clmref}[1]{Claim~\ref{#1}}
\newcommand{\crlref}[1]{Corollary~\ref{#1}}
\newcommand{\eqnref}[1]{Eqn.~\ref{#1}}

\newcommand{\asuref}[1]{Assumption~\ref{#1}}

\newtheorem{remark}{Remark}
\newtheorem{theorem}{Theorem}
\newtheorem{lemma}{Lemma}
\newtheorem{definition}{Definition}

\newtheorem{assumption}{Assumption}

\newtheorem{proposition}{Proposition}

\newcommand{\doo}{\text{do}}

\newcommand\independent{\protect\mathpalette{\protect\independenT}{\perp}}
\def\independenT#1#2{\mathrel{\rlap{$#1#2$}\mkern2mu{#1#2}}}
\title{Improving Self-supervision on Out-of-distribution Robustness through Causal Transportability}
\title{Improving Self-supervision Robustness Through Causal Transportability}
\title{Towards Causality: Unlearning Spurious Correlations From Visual Representations}
\title{Causal Transportability from Neural Representations}
\title{Causal Transportability from Undoing Neural Representations Bias}
\title{Neural Representation Unlearning for Out-of-distribution Robustness}
\title{Causal Transportability Improves Out-of-distribution Robustness}
\title{Undoing non-robust features for Robust Generalization}
\title{Causal Transportability for Visual Recognition}

\newcommand*\samethanks[1][\value{footnote}]{\footnotemark[#1]}

\author{
Chengzhi Mao$^1$\thanks{Equal Contribution.}
\hspace{1em}
Kevin Xia$^1$\samethanks
\hspace{1em}
James Wang$^1$
\hspace{1em} Hao Wang$^2$ \\
\hspace{1em} Junfeng Yang$^1$
\hspace{1em} Elias Bareinboim$^1$\hspace{1em}Carl Vondrick$^1$
\\
$^1$Columbia University\hspace{1em} $^2$Rutgers University\\
{\tt\small \{mcz,kevinmxia,jlw2247,junfeng,eb,vondrick\}@cs.columbia.edu, hoguewang@gmail.com}}

\maketitle

\begin{abstract}
Visual representations underlie object recognition tasks, but they often contain both robust and non-robust features. Our main observation is that image classifiers may perform poorly on out-of-distribution samples because spurious correlations between non-robust features and labels can be changed in a new environment.
By analyzing procedures for out-of-distribution generalization with a causal graph, we show that standard classifiers fail because the association between images and labels is not transportable across settings. However, we then show that the causal effect, which severs all sources of confounding, remains invariant across domains. This motivates us to develop an algorithm to estimate the causal effect for image classification, which is transportable (i.e., invariant) across source and target environments. Without observing additional variables, we show that we can derive an estimand for the causal effect under empirical assumptions using representations in deep models as proxies. Theoretical analysis, empirical results, and visualizations show that our approach captures causal invariances and improves overall generalization.
\end{abstract}


\section{Introduction}
\label{sec:intro}

Visual representations underlie most object recognition systems today  \cite{kolesnikov2020big, chen2020simple, mocov2, radford2021learning}. By learning from large image datasets, convolutional networks have been able to create excellent visual representations that improve many downstream image classification tasks \cite{chen2020simple, mocov2,lin2020shoestring}. However, central to this framework is the need to generalize to new visual distributions  at inference time \cite{objectnet,ImageNetOverfit, hendrycks2019robustness, imagenetbiased, VLCS, PACS, officehome, domainNet, theory_domain_generalize}. 

The most popular technique to use representations is to fine-tune the backbone model or fit a linear model on the target classification task~\cite{kolesnikov2020big}. Although this approach is effective on in-distribution benchmarks, the resulting classifier also inherits the biases from the target dataset.  Given the nature of how data is collected, essentially every realistic image dataset will have spurious features, which will impact the generalization  of computer vision systems. 
Specifically, the learned representation will encode features that correspond to spurious correlations found in the training data.

In this paper, we investigate  visual representations for object recognition through the lenses of causality \cite{pearl:2k,pearl:mackenzie2018, bar:etal2020}. Specifically, we will revisit  the out-of-distribution image classification task through causal-transportability language  \cite{bareinboim2013general,BareinboimTransport,correa2019statistical}, which will allow us to formally model both confounding and structural invariances shared across disparate environments. In our context, we will show how different environments select a distinct set of robust and non-robust features in constructing the input dataset. The training environment may tend to select specific nuisances with the given category, creating spurious correlations between the nuisances and the predicted class. In fact, standard classifiers will tend to use those spurious correlations, which analytically explains why they result in poor generalization performance to novel target distributions \cite{Rendition, sagawa2019distributionally, wang2019learning}. 

First, we will show that the association between image and label is not in generalizable (in causal language, transportable) across domains. We then note that the causal effect from the input to the output, which severs any spurious correlations, is invariant when the environment changes with respect to the features' distributions. This motivates us to pursue to an image classification strategy that will leverage causal effects, instead of merely the association, and will act as an anchor, providing stability across changing conditions and allowing extrapolation to more likely succeed.
Getting the causal effect for natural images is challenging because there are innumerable unobserved confounding factors within realistic data. Under some relatively mild assumptions, we will be able to extract the robust features from observational data through both causal and deep representations  \cite{bareinboim2013transportability, transportability_soft, DA_CI_ICD, inv_causal_rob, rojas2018invariant, OrphicX,wanyuicml21}, and then use the representations as proxies for identifying the causal effect without requiring observations of the confounding factors. 





For both supervised and self-supervised representations, our experimental results show that incorporating the causal structure improves performance when generalizing to new domains. Our method is compatible with many existing representations without requiring re-training, making the approach effective to deploy in practice. Compared to the standard techniques to use representations, our causally motivated approach can obtain significant gain on CMNIST (up to 40\% gain), WaterBird (up to 25\% gain),  ImageNet-Sketch (up to 8\% gain), and ImageNet-Rendition (up to 7\%) datasets. Our work illustrates the importance of causal quantities in out-of-distribution image classification and proposes an effective empirical method that allows the learning of a classifier robust to domain change. Our code is available at \url{https://github.com/cvlab-columbia/CT4Recognition}.


\vspace{-2mm}

\section{Related Work}

\textbf{Causal Inference and Transportability Theory.}   Causal inference provides a principled framework for modeling structural invariances \cite{pearl:2k} and the problem of generalizing, or transporting, across environments and changing conditions \cite{bareinboim2013general,BareinboimTransport,correa2019statistical, bareinboim2013transportability,bareinboim2014transportability,transportability_soft, DA_CI_ICD, inv_causal_rob, rojas2018invariant}. A few image generation works have modeled a causal connection between images and their labels, often assuming the labels are generating the images \cite{rojas2018invariant, heinze2017conditional}, and some prior work studied the connection between causality and specific types of generalizations \cite{arjovsky2020invariant, mao2021generative, causalmatching,yue2021transporting}. Our work studies recognition and reverses this direction, on purpose, since we consider that the images generate the labels through a human-labeling process; this model is detailed in Sec.~3. 
To estimate arbitrary causal effects, one can construct a proxy causal-neural models \cite{xia2021cnc}, but in this paper we focus on directly computing and optimizing a specific causal estimand. Existing work on this often assumes one can intervene on the data \cite{mao2021generative, ilse2020selecting_dataaug_causal}  or observe latent confounding factors \cite{IFL, ilse2020selecting_dataaug_causal}. These assumptions are often overly optimistic for natural images, as image data is passive (preventing intervention) and does not allow us to observe additional confounding factors.




\textbf{Out of distribution Generalization in Vision.}
There are two major types of domain generalization(DG): the multi-source DG and the single-source DG. Multi-source domain generalization has been studied \cite{arjovsky2020invariant, li2018learning, ajakan2014domain, CORAL, blanchard2017domain, zhang2020adaptive}, where the algorithm knows the domain index which the data points are sampled from. A large number of approaches have been proposed to learn classifiers that generalize to out-of-distribution and new environments \cite{VLCS, PACS, officehome, domainNet, Rendition, wang2019learning}. In practice, however, it is often challenging to collect images with accurate domain labels, such as from the Internet. 
Single domain generalization \cite{heinze2017conditional} does not require the domain index assumption, where all training data are assumed to be sampled from the same domain. Still, domain generalization under this setup is more challenging due to lacking the domain information. Existing work achieves generalization via self-supervised learning\cite{JiGen}, anticipating distribution shifting \cite{M-ADA}, creating pseudo domain split\cite{dg_mmld}, adversarial self-challenging \cite{ huang2020self}, and generative data augmentations \cite{mao2021generative}. Recently, the attention operation is also shown to be effective for improving robustness\cite{dosovitskiy2020image, paul2021vision, mao2021discrete}. However, a principled framework for modeling generalization to new environments is still missing.

\section{Problem Formulation -- Image Recognition Through Causal Lenses}

We start by grounding the problem of image recognition in a causal framework to illustrate the key challenges of out-of-distribution generalization compared to its in-distribution counterpart.

\subsection{Structural Modeling of the Classification Task}

\begin{figure}[t]
  \centering
  \includegraphics[width=1\linewidth]{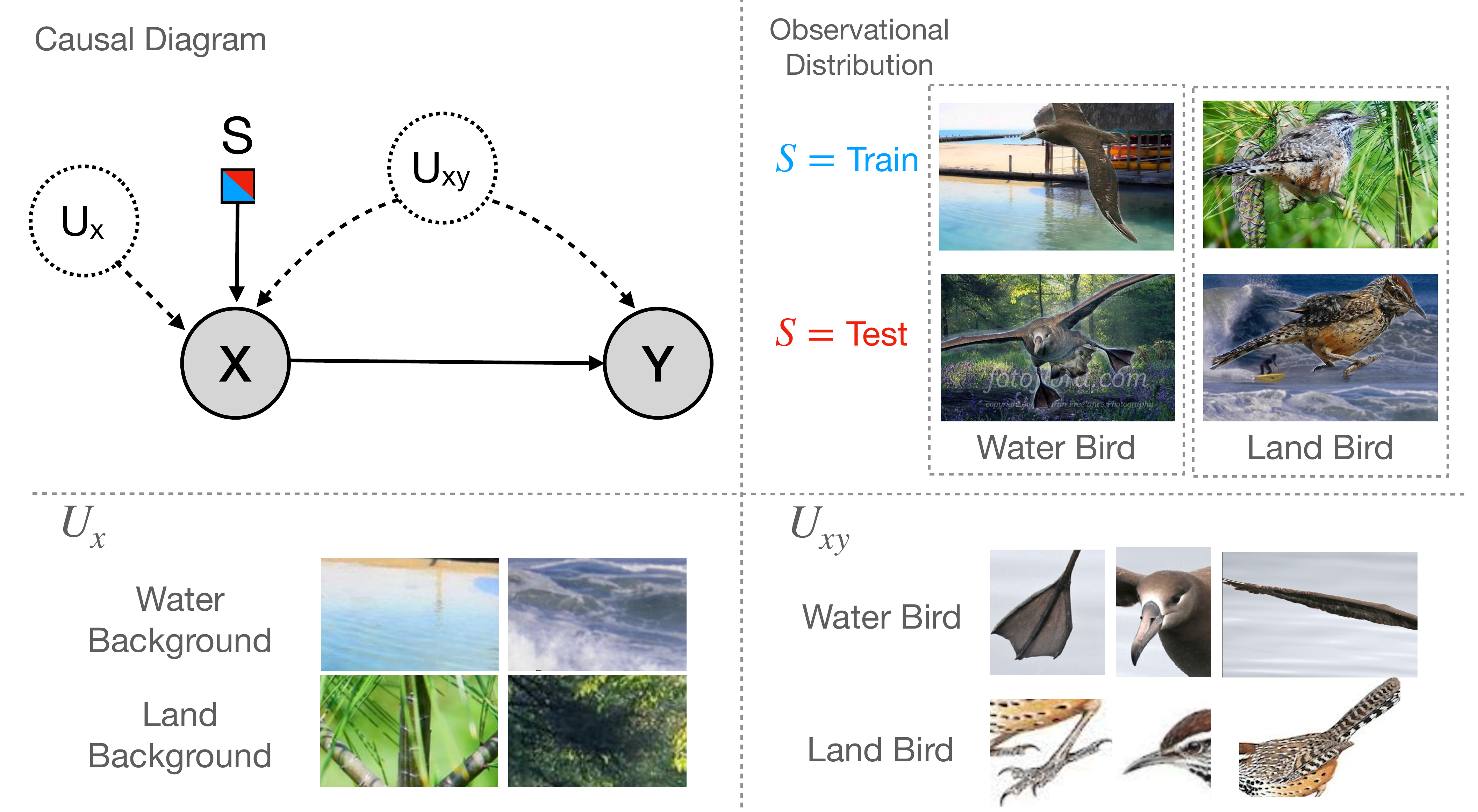}
  \vspace{-6mm}
  \caption{Causal graph for out-of-distribution image classification (top left). Image $X$ is constructed from nuisance features $U_X$ (bottom left) and concept features $U_{XY}$ (bottom right). Label $Y$ is created from $X$ and $U_{XY}$. $S$, the transportability node, points to nodes with changes between domains, where $X$ combines `waterbird' with `water background' during the training ($S = 0$) and `water bird' with `land background' at testing ($S = 1$) (top right).}
   \label{fig:transportability}
   \vspace{-5mm}
\end{figure}

Let the pair $X, Y$ represent the random variables related to images and their labels, and $x, y$ the specific instantiations of the pixels and label. 
Given an input image $X=x$, the goal of the image classification task is to predict its label, $Y=y$. Taking a probabilistic interpretation, a standard strategy is to train a model to learn $P(Y \mid X)$ given data points of $X=x$ and $Y=y$, and then choose a class at inference time via $\argmax_y P(Y=y \mid X=x)$.

We will take a causal approach here, and model the underlying generative process of $X$ and $Y$ using causal semantics. Specifically, we will use a class of generative  processes known as a \textit{structural causal model} (SCM, for short) \cite[Ch.~7]{pearl:2k}. Each SCM $M$ encodes a 4-tuple $\langle V = \{X, Y\}, U = \{U_X, U_{XY}\}, \mathcal{F} = \{f_X, f_Y\}, P(U) \rangle$, where $V$ is the set of observed variables, in this case, the image ($X$) and its label ($Y$); $U$ represents unobserved variables encoding external sources of variation not captured in the image and the label themselves (more details next); $\mathcal{F}$ is the set of mechanisms $\{f_X, f_Y\}$, which determine the generative processes of $X$ and $Y$ such that $X \gets f_X(U_X, U_{XY})$ and $Y \gets f_Y(X, U_{XY})$; $P(U)$ represents a probability distribution over the unobserved variables. 

In particular, we call $U_{XY}$ the ``concept vector'', as it represents all underlying factors that produce both the core features of the object in image $x$ and its label, $y$. For example, one instantiation of $U_{XY} = u_{XY}$ may encode the concepts of ``flippers'' and ``wing,'' which are translated into an image of a ``waterbird'' when passed into $f_X$. $U_X$ represents nuisance factors, such as the background, that affect the generation process of the image. Likewise, $f_Y$ may represent someone who is labeling image $x$ and will have a conceptual understanding of waterbird through $u_{XY}$. One natural, albeit critical observation, is that if $f_X$ selects the color ``flippers'' and the background ``water'' more likely together, there would be a strong association between these two concepts, given the image.  Together, the underlying distribution over $P(U_{XY}, U_X)$ combined with functions $f_X$ and $f_Y$ induce a distribution over $P(X, Y)$, which is how the data is generated. The SCM $M$ is almost never observable, and it is in general, in a formal sense, impossible to recover the structural functions ($\mathcal{F}$) and probability over the exogenous variables ($P(U)$) from observational data alone ($P(V)$) \cite[Thm.~1]{bar:etal2020}. 

\subsection{Modeling In vs.~Out-of-Distribution Generalization through Transportability}

When training a classifier for in-distribution problems, both training and test data come from the same domain. In the out-of-distribution case, also known as the \emph{transportability} problem in the causal inference literature \cite{bareinboim2013general,BareinboimTransport,correa2019statistical}, training data may come from a domain $\pi$ that differs from the test domain, $\pi^*$. We assume that the labeling process and underlying concepts are consistent across domains (i.e. $f_Y$ and $P(U_{XY})$ remain the same in both settings), but the generative process of the image $X$ may change (i.e. $f^*_X$ and $P^*(U_X)$ may differ from $f_X$ and $P(U_X)$, respectively).

In general, we do not know the true underlying mechanisms $f_X$, $f^*_X$, and $f_Y$, nor can we observe the immeasurably large space of $P(U_X, U_{XY})$. However, we can represent the structural invariances across domains by leveraging a graphical representation shown in Fig.~\ref{fig:transportability}. The disparities across domains $\pi$ and $\pi^*$ are usually modeled by a transportability node called $S$ \cite{BareinboimTransport}, which can be interpreted as a switch across domains; i.e., $f_X$ will be active if $S = 0$, and $f^*_X$ otherwise.  For concreteness, consider two different categories of birds, the waterbird and the landbird, between which we want to discriminate. Both bird categories have their own underlying features $U_{XY}$ that cause an annotator to label them as a waterbird or landbird. However, while waterbirds are typically paired with water backgrounds in images generated in the source domain ($S = 0$), this factor may change in the target domain ($S = 1$), where waterbirds are now commonly shown in land backgrounds.

\begin{figure}[t]
  \centering
  \includegraphics[width=0.8\linewidth]{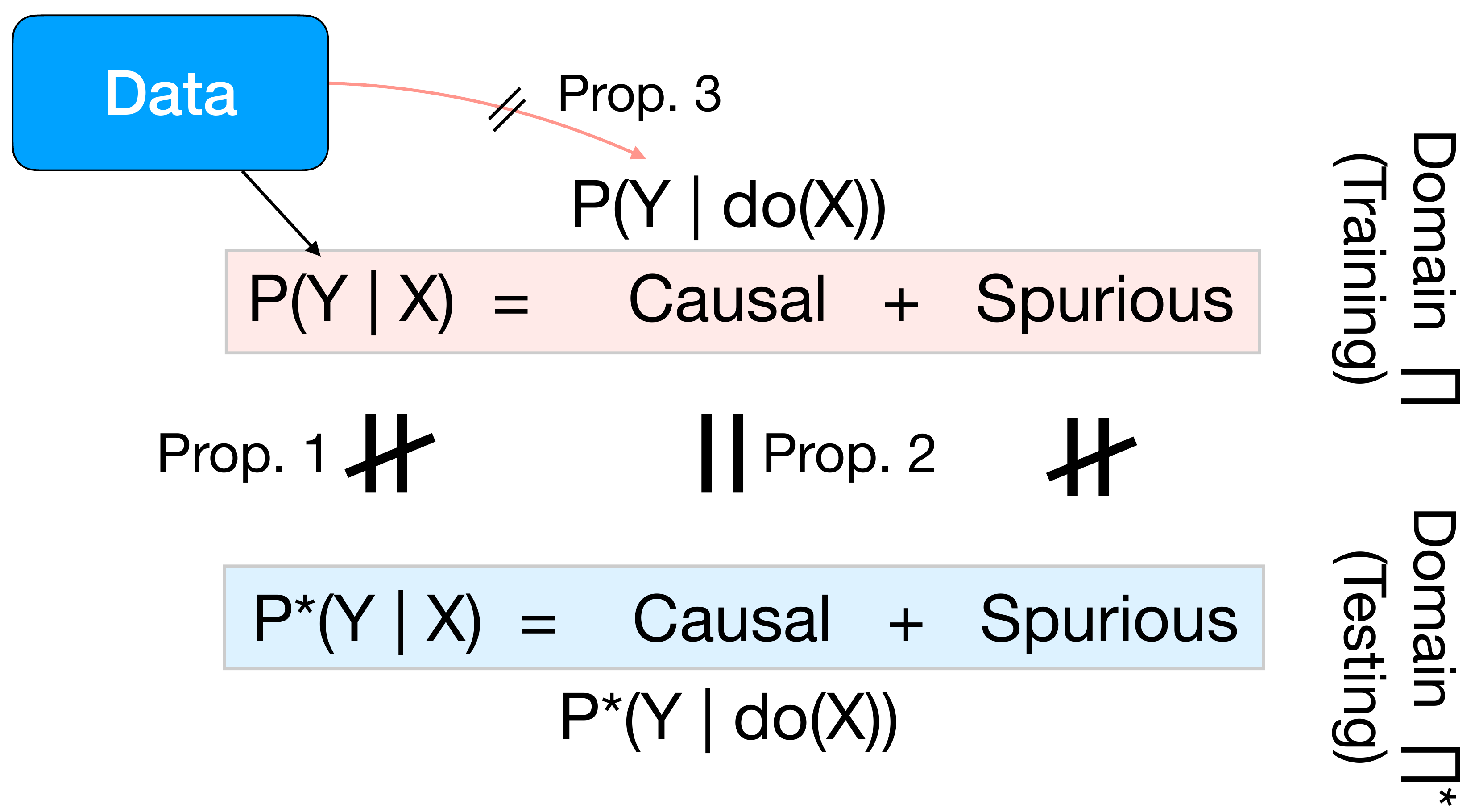}
  \vspace{-3mm}
  \caption{Visualization comparing quantities between domains $\pi$ and $\pi^*$. Prop.~\ref{prop:assoc-transport} shows that $P(Y \mid X)$, which contains both causal and spurious information, does not match $P^*(Y \mid X)$.  Prop.~\ref{prop:causal-transport} shows that the causal effect is invariant across settings, i.e., $P(Y \mid \doo(X)) = P^*(Y \mid \doo(X))$. However, Prop.~\ref{prop:id} shows that unlike $P(Y \mid X)$, $P(Y \mid \doo(X))$ is not identifiable from  $\pi$-data .}
   \label{fig:propositions}
   \vspace{-5mm}
\end{figure}

In the in-distribution case, the more traditional strategy of learning $P(Y | X)$ is logical, in the sense that it leverages all possible information to maximize the chance of predicting the correct label. However, given the way the data generation process is modeled, it is easy to see why this same strategy fails in the out-of-distribution case. Since only data from domain $\pi$ is given, we can only train a model on $P(Y \mid X)$, which does not adequately model $P^*(Y \mid X)$.

\begin{proposition}
    \label{prop:assoc-transport}
    Let $M$ and $M^*$ be the two underlying SCMs representing the source and target domains, $\pi$ and $\pi^*$, and compatible with the assumptions represented in the causal graph in Fig.~\ref{fig:transportability}. Then, $P^*(Y \mid X) \neq P(Y \mid X)$.
\end{proposition}

In words, the classifier represented by the quantity $P(Y \mid X)$, in $\pi$, is not \emph{transportable} across settings and cannot be used to make statements about $P^*(Y \mid X)$, even when everything aside from the mechanism of $X$ ($f_X$) remains invariant (including the labeler $f_Y$). Intuitively, this is due to the unobserved confounding, or spurious effects, between $X$ and $Y$ through $U_{XY}$. By conditioning on $X$, the variables $Y$ and $S$ become d-connected via the path through $U_{XY}$, i.e. $P(Y \mid X, S = 0) \neq P(Y \mid X, S = 1)$. This result is also shown pictorially in Fig.~\ref{fig:propositions}. 

In addition to the spurious effects,  $X$ and $Y$ still co-vary due to the direct link $X \rightarrow Y$. In other words, the labeling process can be seen as moving unobserved co-variation that goes through $U_{xy}$ to the observed link  $X \rightarrow Y$. These variations are known as the causal effect of $X$ on $Y$. Intuitively, one can think of the causal effect $P(Y \mid \doo(X))$ as describing the interventional world where arrows towards $X$ can be thought of as removed. This includes the $S$-node, which no longer has an influence on $X$ when $X$ is forced to take a certain value, say $x$. This is promising since if a quantity is not affected by $S$, that implies that it is invariant across domains. As shown next, this is indeed the case with $P(Y \mid \doo(X))$.

\begin{proposition}
    \label{prop:causal-transport}
        Let $M$ and $M^*$ be the two underlying SCMs representing the source and target domains, $\pi$ and $\pi^*$, and compatible with the causal graph in Fig.~\ref{fig:transportability}. Then, $P^*(Y \mid \doo(X)) = P(Y \mid \doo(X))$.
\end{proposition}

Regardless of the change in the mechanism of $f^*_X$ and $P^*(U_X)$, it is guaranteed that the causal effect of $X$ on $Y$ will remain invariant across $\pi$ and $\pi^*$. In causal language, $P^*(Y \mid \doo(X))$ is transportable across settings.

\subsection{Identifiability}

Given that the causal effect is invariant across domains, we consider using $P(Y \mid \doo(X))$ as a surrogate for $P^*(Y \mid X)$ for classification purposes (out-of-distribution), instead of the classifier trained in the source, $P(Y \mid X)$. That leaves the question of how to identify (and then estimate) this quantity given observational data, $P(X, Y)$. Unfortunately, this is still not possible in the general case.

\begin{proposition}
    \label{prop:id}
    Let $M$ be the SCM representing domain $\pi$ and described through the causal diagram $G$ in Fig.~\ref{fig:transportability}. The interventional distribution $P(Y \mid \doo(X))$ is not identifiable from $G$ and the observational distribution $P(X, Y)$.
\end{proposition}

In words, non-identifiability suggests that there are multiple SCMs that are consistent with $P(X, Y)$ and that induce different distributions $P(Y \mid \doo(X))$. This means that $P(X, Y)$ is too weak, in some sense, and it is too under-specified to allow one to deduce $P(Y \mid \doo(X))$. Additional assumptions are needed to identify (and then estimate) this causal effect.

In fact, some prior work has assumed that \textit{all} back-door variables can be observed~\cite[Sec.~3.3.1]{pearl:2k}, which means that all the variations represented originally in the unobserved confounder $U_{xy}$
are, in some sense, captured by the model. When additional domain index information is available (e.g. styles of the images), prior works such as IRM \cite{arjovsky2020invariant}, MLLD \cite{dg_mmld}, and DANN \cite{ajakan2014domain} have performed adjustment-like operations with the domain index. In most image datasets that contain only images and their labels, the assumption that all back-door variables (and sources of co-variation) are observable is overly stringent. Even when additional data is available, it is unlikely that such data contains all possible variations encapsulated by the concept vector. Our goal now is to identify the effect of $X$ on $Y$ without having knowledge of the back-door variables. 


\section{Neural Representation Approach to Deriving a Causal Estimand}

Following the previous understanding that  $P(Y \mid \doo(X))$ is a suitable proxy for the classifier in the target domain, $P^*(Y | X)$, we discuss in this section sufficient assumptions that would allow us to estimate such a quantity.  Further, we discuss methods that could allow the practical realizability of these assumptions in the context of image recognition.


To realize the goal of estimating the target causal effect,  we build two neural network models: $\hat{P}(R \mid X)$, which generates visual representations $R$ from images $X$, and $\hat{P}(Y \mid R, X)$, which uses both $R$ and $X$ to classify $Y$. We make the following  assumptions about the structure of image $X$ and the properties of these networks:

\begin{assumption}[\textbf{Decomposition}]\label{assu:d1}
Each image $X$ can be decomposed into causal factors $Z$ and spurious factors $W$ (i.e. $X = (Z, W)$), and the generative process follows the causal graph in Fig.~\ref{fig:frontdoor}.
\end{assumption}

One may be tempted to surmise that this is an innocent assumption, but it does make strong claims about the generative process. The interpretation is that $W$ contains all of the lower level signals or patches of the image, which may contain concepts confounding with $Y$. On the other hand, $Z$ refines these patches into interpretable factors, which is what is visually used by the labeler. Since $Z$ is a direct function of $W$, these factors are not confounded. For example, while $W$ might include various pieces of information such as patches of blue in the water or texture of feathers, $Z$ refines all of these signals into factors such as ``waterbird shape,'' which is then used by the labeler to choose ``waterbird'' for $Y$. While this assumption may not be true in all settings, we believe that many practical, image settings can be approximated by this assumption.



\begin{figure}
\centering
  \includegraphics[width=0.7\linewidth]{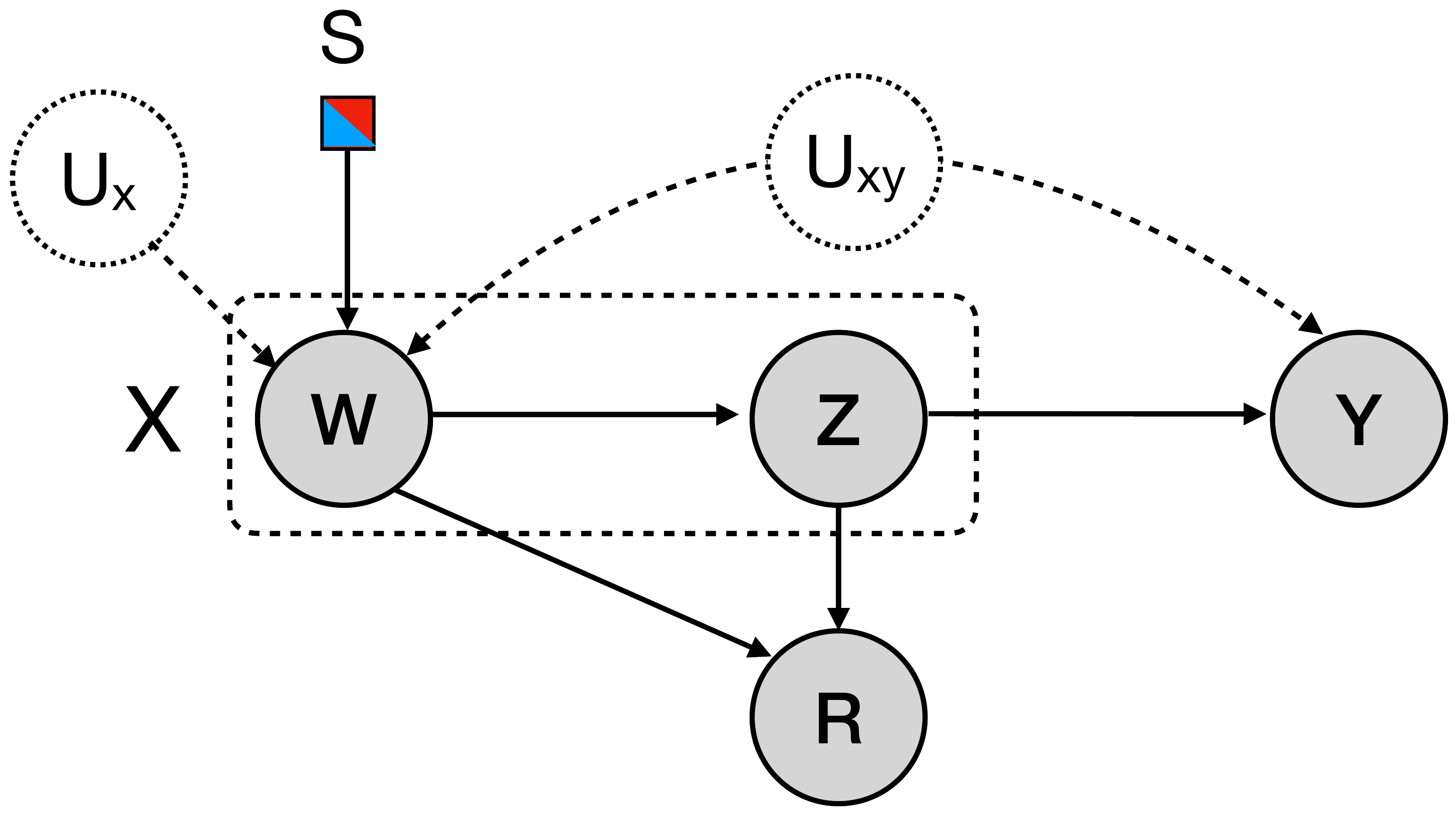}
  \vspace{-4mm}
  \caption{Expanded causal model with decomposition of image $X$ and representation $R$. Gray nodes denote observed variables.}  \label{fig:frontdoor}
  \vspace{-5mm}
\end{figure}

\begin{assumption}[\textbf{Sufficient representation}]\label{assu:s2}
The neural representations $R \sim \hat{P}(R \mid Z, W)$ are learned such that they do not lose information w.r.t. $Z$. In words, for two samples $r_1$ and $r_2$ from $\hat{P}(R \mid z_1, w_1)$ and $\hat{P}(R \mid z_2, w_2)$, respectively, $r_1 \neq r_2$ if $z_1 \neq z_2$.
\end{assumption}

This is a somewhat more technical assumption, which says that the neural representation has enough capacity to represent unambiguously the causal factors.  
This assumption should hold in general given a proper choice of model for $\hat{P}(R \mid X)$, which we further elaborate in Sec.~\ref{sec:FDconstruction}.



\begin{assumption}[\textbf{Selective prediction}]\label{assu:s3}
Consider two images of $X$, $x = (z, w)$ and $x' = (z', w')$, with neural output $\hat{P}$, and the true labeling probability $P$. Let $R=r$ be a representation of $x$, sampled from $\hat{P}(R \mid x)$. Then, $\hat{P}(Y=y \mid R=r, X=x') = P(y \mid z, w')$. 
\end{assumption}

The details on how to select the specific architectural design for constructing $\hat{P}(Y \mid R, X)$ that satisfies this assumption is discussed in more detail in Sec.~\ref{sec:PYconstruction}. Still, in words, the assumption says that once inputted with two images  $x$ and $x'$ ($x$ in its representation form, $r$), the network will make the same prediction $y$ as if it were the true labeler when inputted with the causal feature $z$, from the first image, and the spurious feature $w'$, from the second image.



Putting all these observations together, we now state one of the main results of the paper:

\begin{theorem}[\textbf{Causal Identification}]\label{thm:causal-fd}
Given the assumptions about the generative process encoded in the causal graph in Fig.~\ref{fig:frontdoor} together with assumptions 1, 2, 3, the causal effect can be computed using neural representation $R$ via  $P(Y=y|do(X=x))=\sum_r \hat{P}(r|x) \sum_{x'} \hat{P}(y|r, x')P(x')$.
\end{theorem}

\begin{proof}
We first derive the following steps.
\begin{align*}
    & P(y \mid \doo(x)) \\
    &= P(y \mid \doo(z, w)) & \text{Assumption~\ref{assu:d1}} \\
    &= P(y \mid \doo(z)) & \text{Do-Calculus Rule 3 \cite{pearl:2k}} \\
    &= \sum_{w'} P(y \mid z, w') P(w') & \text{Backdoor Criterion} \\
    &= \sum_{z', w'} P(y \mid z, w')P(z', w') & \text{Marginalization} 
\end{align*}
By Assumptions~\ref{assu:s2} and \ref{assu:s3}, the last expression can be re-written as 
\begin{align*}
&= \sum_{x'} \hat{P}(y \mid r, x'=(z', w'))P(x')
\end{align*}
where $r$ is sampled from $\hat{P}(R \mid x)$. Since Assumption~\ref{assu:s3} applies for any sampled value of $R$, we can average across samples of $R$,
\begin{align*}
&= \sum_{r} \hat{P}(r \mid x) \sum_{x'} \hat{P}(y \mid r, x')P(x'),
\end{align*}
concluding the proof. 
\end{proof}

\begin{algorithm}[t]
\caption{Causal-Transportability Model Training}
\label{algorithm: FDtrain}
\begin{algorithmic}[1]
\STATE {\bfseries Input:} Training set $D$ over $\{(X, Y)\}$.
\STATE{\textbf{Phase 1:} Compute $\hat{P}(R|X)$ from representation of VAE or pretrained model.}

\STATE{\textbf{Phase 2:}}
\FOR{$i=1,...,K$}
\STATE{Sample $x_i, r_i, y_i$ from the joint distribution $D'=(X,R,Y)$}
\STATE{Random sample $x_i'$ from the same category as $x_i$}
\STATE{Train $\hat{P}(Y|X', R)$ via minimizing the classification loss $\mathcal{L}$ through gradient descent.}
\ENDFOR
\STATE {\bfseries Output:} Model $\hat{P}(R|X)$ and $\hat{P}(Y|X,R)$


\end{algorithmic}
\end{algorithm}

\begin{algorithm}[t]
\caption{Causal-Transportability Effect Evaluation}
\label{algorithm: FDinfer}
\begin{algorithmic}[1]
\STATE {\bfseries Input:} Query $x$, training distribution $D$ over $\{(X, Y)\}$, model $\hat{P}(R|X)$ and $\hat{P}(Y|X', R)$, the sampling time $N_i$ for the representation variable $R$, and the sampling time $N_j$ for $X'$.
\FOR{$i=1,...,N_i$}
\STATE{$\r_i\leftarrow \hat{P}(r|x)$}
\FOR{$j=1,...,N_j$}
\STATE{Random sample $\x_{ij}'$ from Training Distribution $D$.}\STATE{Compute $\hat{P}(Y|x'_{ij}, r_i)$}
\ENDFOR
\ENDFOR
\STATE{Calculate the causal effect $P(y|\doo(X=x)) = \sum_i \hat{P}(r_i|x) \sum_{j} \hat{P}(y|r_i, x_{ij}')P(x_{ij}')$}
\STATE {\bfseries Output:} Class $\hat{y}=\text{argmax}_y P(y|\doo(X=x))$.
\end{algorithmic}
\end{algorithm}


The intuition behind this derivation is that if the image $x$ can be decomposed into causal factors ($z$) and spurious factors ($w$), as shown in Fig.~\ref{fig:frontdoor}, then the causal effect is isolated in $z$, and $w$ can be ignored. By conditioning on $W=w'$, using another image, all the backdoor paths from $Z$ to $Y$ are  blocked, which leads to an identifiable result (i.e., without do-terms). That leaves the question of how to obtain the $z$ component from image $x$, and $w'$ from $x'$. The general idea behind assumptions \ref{assu:s2} and \ref{assu:s3}, and the last two lines of the derivation, is that $\hat{P}(Y \mid R, X)$ is able to extract all of the causal information $z$ from the representation $r$, and extract the spurious information $w'$ from the second image $x'$, which will happen through the design of the neural net.


Altogether, Theorem~\ref{thm:causal-fd} allows us to estimate the causal effect through \footnote{Interestingly, the derivation of this expression is somewhat similar to the well-known identification strategy named the front-door criterion ~\cite[Sec.~3.3.2]{pearl:2k}. One of the key assumptions made by the front-door is that there exists a variable $M$ that acts as an (unconfounded) mediator between $X$ and $Y$. In spirit, $R$, our deep representation, resembles $M$. Despite the syntactical appearances, the variable $R$ in the case here is not exactly a mediator, in the original sense, since it acts as a proxy for both $X$ and $Z$. }:
\begin{eqnarray}\label{eq:fd}
P(y | \doo(X=x)) = \sum_r \hat{P}(r | x) \sum_{x'} \hat{P}(y | r, x') P(x')
\end{eqnarray}

To use this formula, we need to construct the neural models to satisfy the three assumptions and properly estimate $P(X)$, $\hat{P}(R|X)$, and $\hat{P}(Y|X,R)$. The term $P(X)$ is straightforward to calculate because we can assume it is sampled from a uniform distribution~\cite{shalizi2013advanced}. The other terms, however, require a more careful construction so as to satisfy the aforementioned assumptions, which are discussed in the following sections.

\subsection{Constructing  \boldmath$P(R|X)$\unboldmath}\label{sec:FDconstruction}

We discuss some classes of models that are valid ways of estimating $\hat{P}(R|X)$ while satisfying Assumption~\ref{assu:s2}.




\begin{table}[t]
\scriptsize
\centering
\begin{tabular}{l|c|c}
         \toprule
         & \multicolumn{2}{c}{Test Accuracy} \\
         &  In-distribution  &  Out-of-distribution\\
         \midrule  
         Chance & 10.0\% & 10.0\%  \\
         ERM \cite{vapnik1992principles} & 99.5\% & 8.3\%  \\
         IRM* \cite{arjovsky2020invariant} & 87.3\% & 18.5\% \\
         
         RSC \cite{huangRSC2020} & 96.6\% & 20.6\% \\
         GenInt \cite{mao2021generative} & 58.5\%& 29.6\%  \\
         \midrule
         Ablation & 97.4\% & 38.8\% \\
         Ours & 82.9\% & \textbf{51.4\%} \\
         \bottomrule
    \end{tabular}
\caption{\small{Accuracy on the CMNIST dataset. Our method advances the state-of-the-art GenInt \cite{mao2021generative} method by over 20\% on the out-of-distribution test set. }} \label{tab:colorfulmnist}
\vspace{-5mm}
\end{table}

\textbf{Variational Auto-Encoder} (VAE) \cite{kingma2014autoencoding} is an unsupervised representation learning approach, which aims to estimate a latent distribution $R$ that can faithfully generate the input distribution. It maximizes the evidence lower bound for the distribution of $X$: $\mathcal{L} = -D_{KL}(q_E(r|x^{(i)})||p_{\theta}(r)) + \mathrm{E}_{q_E(r|x^{(i)})}[\log{p_{\theta}(x^{(i)}|r)}]$, where $E$ is the encoder in the VAE.  
As VAEs are optimized to reconstruct input images via the term $\mathrm{E}_{q_E(r|x^{(i)})}[\log{p_{\theta}(x^{(i)}|r)}]$, the representation $R$ should contain all the causal information from the input images, satisfying Assumption~\ref{assu:s2}.




\textbf{Constrastive Learning} is another unsupervised learning approach that produces representations that can align views of the same image while separating views of different images. 
Given enough negative examples, contrastive learning will produce representations that are invariant under data augmentation, which still maintains all causal information from the input images, also satisfying Assumption~\ref{assu:s2}.



\textbf{Pretrained models from larger dataset.} Empirically, deep neural networks show better generalization when pretrained from large datasets. This suggests that their representation $R$ does not drop robust features for classification and keeps the information about $Z$, satisfying Assumption~\ref{assu:s2}.

\subsection{Constructing \boldmath$P(Y|R, X)$\unboldmath}\label{sec:PYconstruction}
To properly evaluate Eq.~\ref{eq:fd}, we also need to estimate a $\hat{P}(Y|R, X)$ such that Assumption~\ref{assu:s3} is satisfied. We discuss some neural network designs to achieve this.

\textbf{Model Design for \boldmath$\hat{P}(Y|R, X)$\unboldmath.} In addition to the representation $R$, we use as input a bag of patches, which are subsampled from input image $X$ into the branch that takes the input $X$. A bag of image patches corrupts the global shape information and often contains local features that are spurious, such as color, texture and background~\cite{mao2021discrete}. During training, the causal features $Z$ in the image tend to be ignored by the read-out model. Specifically, we have $\hat{P}(Y|R \sim \hat{P}(R \mid Z, W), X=(Z, W)) = \hat{P}(Y|R \sim \hat{P}(R \mid Z, W), W)$. During training, the image $X$ and the representation $R$ are sampled from the same instance. During testing, the image $X$ can be sampled from an arbitrary instance.

The model $\hat{P}(Y|R, X)$ has limited capacity. Given that the model has learned the information about $W$, learning $W$ from $R$ again will not further decrease the empirical loss.  Thus, the model will learn $Z$ from the representation $R$ and ignore the $W$ from the representation. In addition,  The pretrained representations $R$, such as the ones from contrastive learning, can reduce the (labeled) sample complexity on classification tasks \cite{arora2019theoretical} than on raw image input, which allows the model to learn $Z$ from $R$ efficiently. This satisfies Assumption~\ref{assu:s3}.

By limiting the capacity of  $\hat{P}(Y|R, X)$, the model tends to use low-level features from the input images $X$ while using high-level deep features from the latent representation $R$. Traditional correlation-based approaches only use $\hat{P}(Y|R)$, which can also include spurious features such as the texture and backgrounds from the representation $R$. With our approach, the low-level spurious features tend to be learned by the model that conditions on the input $X$, and the model will discard those features after marginalizing over the variable $X$.

\begin{table}[t]
\centering
    \centering
    \scriptsize
    \begin{tabular}{l|c|c|c|c}
         \toprule
         Method & Domain ID& Train  & I.I.D & OOD \\
         \midrule
         
        GDRO* \cite{GDRO} & Yes & 100.0\% & \textbf{97.4\%} & 76.9\% \\
        \midrule
        ERM & No &100.0\% & 97.3\% & 52.0\%\\
        RSC & No & 92.2\% & 95.6\% & 49.7\% \\
        Ablation & No & 99.4\% & 96.8\% & 71.6\%\\ 
         Ours & No & 99.4\% & 96.8\% & \textbf{77.9\%} \\ 
         
         \bottomrule
    \end{tabular}
\caption{Accuracy on the WaterBird dataset. Our causal method improves ERM model's worst group OOD generalization significantly. Our approach achieves performance on par with group invariant training (GDRO) without needing the domain index.} \label{tab:waterbird}
\vspace{-2mm}
\end{table}

\begin{table}[t]
\scriptsize
\centering
\begin{tabular}{l|c|c|c}
         \toprule
         & \multicolumn{3}{c}{OOD Test Accuracy} \\
         &  Moco-v2  &  SWAV & SimCLR\\
         \midrule  
         ERM \cite{vapnik1992principles} & 14.59\% & 20.00\% & 27.73\%\\
         Ablation & 17.04\% & 20.25\% & 28.44\% \\
         Ours & \textbf{18.02\%} & \textbf{20.42\%} & \textbf{29.41\%}\\
         \bottomrule
    \end{tabular}
\caption{\small{Accuracy on the Imagenet-9 adversarial backgrounds.}} \label{tab:i9}
\vspace{-5mm}
\end{table}

\begin{table*}
\centering
\scriptsize
    \centering
    \begin{tabular}{l|c|c|c|c||c|c|c|c}
         \toprule
          & \multicolumn{4}{c||}{ImageNet Rendition} & \multicolumn{4}{c}{ImageNet Sketch} \\
          Algorithm &  ERM  & RSC  &  Ablation  & Ours & ERM  & RSC  & Ablation & Ours\\ 
         \midrule
         Moco-v2 & 26.92\% & 26.14\% &   25.96\% & \textbf{28.70\%} &  17.29\%  & 16.43\% &   14.11\% & \textbf{19.09\% }\\
         SWAV & 31.77\% &  30.47\% &  30.32\% & \textbf{33.32\%} &  21.51\% & 21.03\%   & 17.26\% & \textbf{22.48\%} \\
        SimCLR & 37.82\%  & 34.06\% & 35.74\% & \textbf{38.25\%} &  27.43\% & 19.26\%   & 24.90\% & \textbf{29.51\%} \\
        \midrule
        ResNet50 & 25.02\% & \textbf{33.34\%}   & 30.96\% & 32.22\%  & 14.45\% & 22.54\%  &  19.19\% & \textbf{22.57\%}  \\
        ResNet152 & 30.53\% & \textbf{37.86\%}   & 34.94\% & 36.07\% &  18.53\% &  26.60\%  & 24.61\% &\textbf{27.07\%}\\
        ResNet101-2x & 31.44\% & 35.50\% &   35.82\% & \textbf{36.70\%} &  19.92\% & 26.38\% &   25.07\% & \textbf{27.41\%} \\ 
        
         \bottomrule
    \end{tabular}
\caption{\small{Robust accuracy on ImageNet-Rendition and ImageNet-Sketch. For contrastive learning based representations, our model achieves improved robustness than standard ERM and the state-of-the-art RSC approach. On superivsed learning representations, the representation may fail to capture all the causal information, where RSC method out-performs ours on two variants on ImageNet Rendition. Overall, our method improves robustness by estimating the causal effect from the representation.}} \label{tab:ImageNetOOD}
\vspace{-5mm}
\end{table*}

\begin{figure*}
    \centering 
\begin{subfigure}{0.22\textwidth}
  \includegraphics[width=\linewidth]{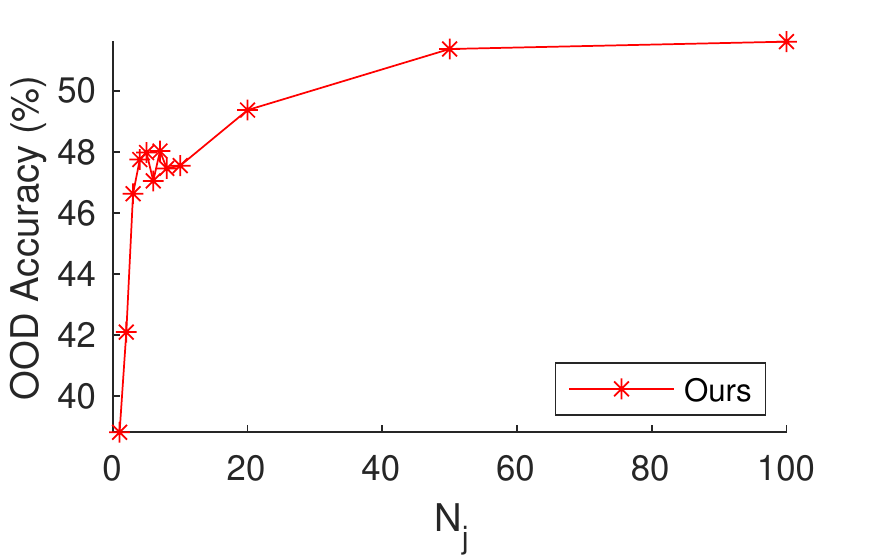}
  \caption{CMNIST, $K=10$}
  \label{fig:birdonpole}
\end{subfigure}
\begin{subfigure}{0.22\textwidth}
  \includegraphics[width=\linewidth]{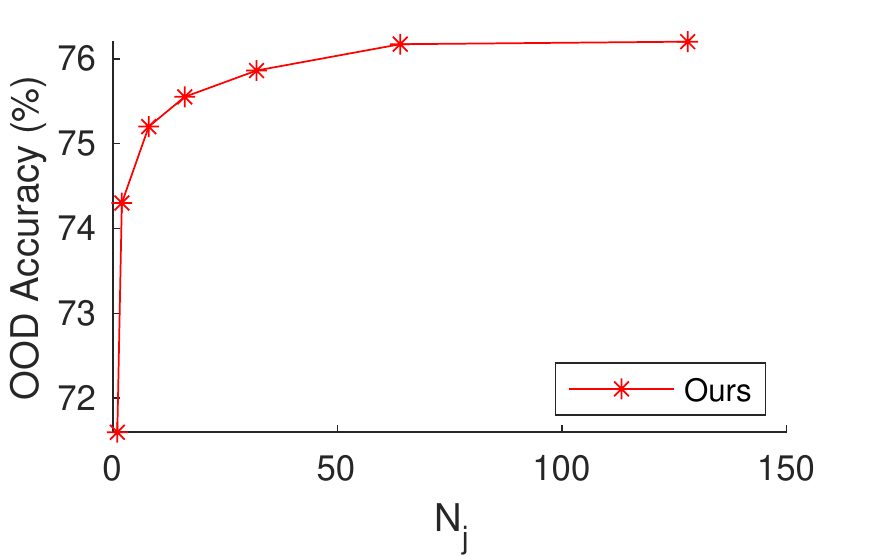}
  \caption{WaterBird, $K=2$}
  \label{fig:segmap}
\end{subfigure}
\begin{subfigure}{0.22\textwidth}
  \includegraphics[width=\linewidth]{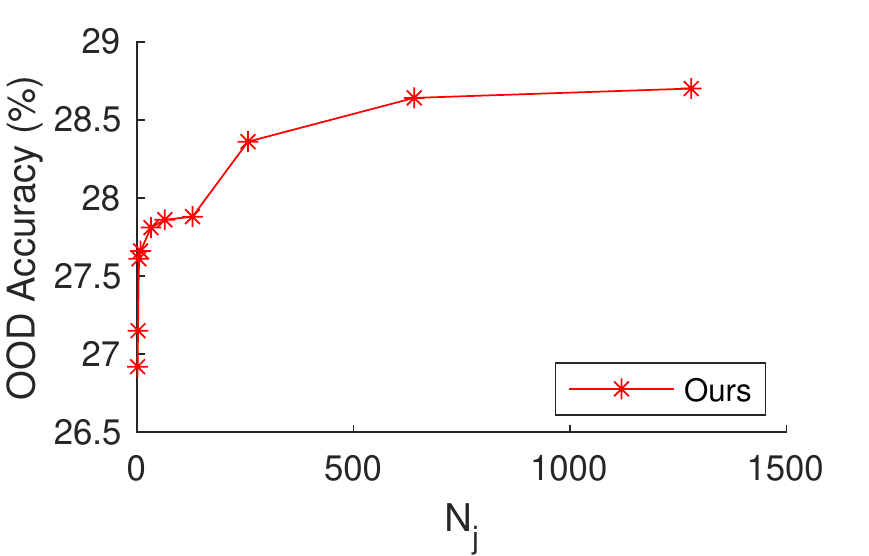}
  \caption{Rendition, $K=1000$}
  \label{fig:segmap}
\end{subfigure}
\begin{subfigure}{0.22\textwidth}
  \includegraphics[width=\linewidth]{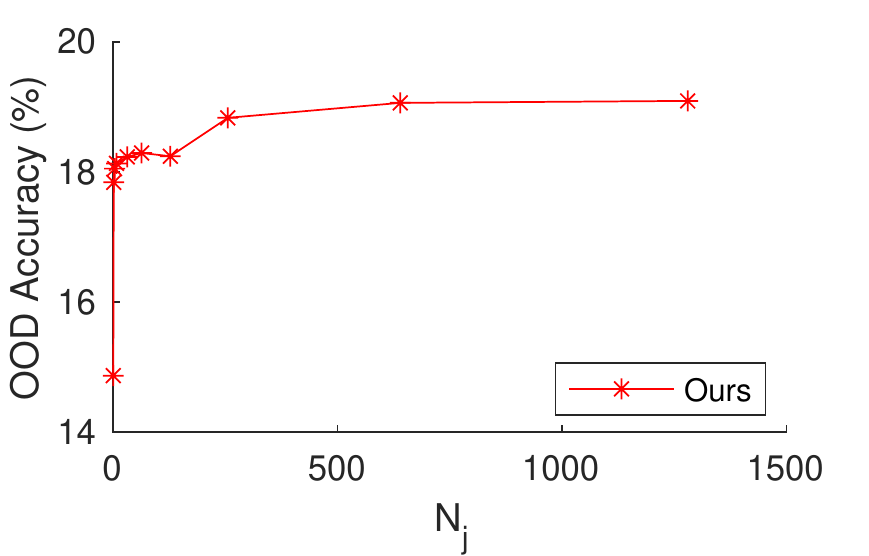}
  \caption{Sketch, $K=1000$}
  \label{fig:segmap}
\end{subfigure}
\vspace{-3mm}
\caption{OOD generalization accuracy under different number of $N_j$. At inference time, by increasing $N_j$ that samples more images $X'$, OOD generalization improve because the spurious correlation is better removed through our approach.}
\vspace{-4mm}
\label{fig:ablation_nj}
\end{figure*}

\begin{figure*}
\centering
  \includegraphics[width=0.9\linewidth]{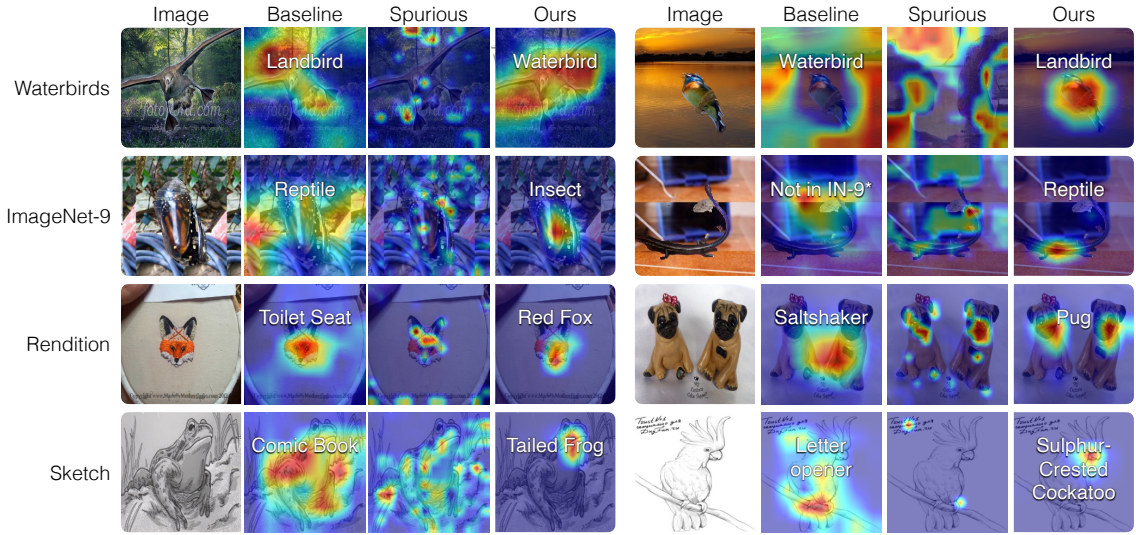}
  \vspace{-3mm}
  \caption{We visualize the input regions that the models use for prediction. We use GradCAM \cite{selvaraju2017grad} and highlight the the discriminative regions that the model relies on with red. The white text shows the model's prediction. The correlation based ERM method often attends to spurious background context. By marginalizing over the spurious features (visualized in the Spurious column), our model captures the right, causal features, which predict the right thing for the right reason.}  \label{fig:gradcam}
  \vspace{-4mm}
\end{figure*}

\subsection{Algorithm}
We describe our training procedure in Algorithm \ref{algorithm: FDtrain}. In the first phase, we estimate $\hat{P}(R|X)$, where we either train a representation with our proposed VAE or contrastive learning approach, or we use representations from a pretrained deep model. In the second phase, we train $\hat{P}(Y|X,R)$ where we sample random images $X$ from the same category as the representation $R$.
We describe our inference procedure in Algorithm \ref{algorithm: FDinfer}, where we infer the $P(y|\doo(X=x))$. We first randomly sample $R$. Then, for each $R$, we sample images $X$ from random categories. Finally, we make the prediction through Theorem~\ref{thm:causal-fd}.




\vspace{-2mm}
\section{Experiment}

\subsection{Datasets}
\textbf{CMNIST.} We use the more challenging setup of colored MNIST dataset with 10 categories \cite{mao2021generative}. The function $F_X(U_x, U_{xy})$ will combine digits with different background colors from the training domain, creating an out-of-distribution (OOD) dataset.
\textbf{WaterBird} dataset \cite{GDRO} contains two classes of foreground birds, the waterbird and the landbird, and two types of backgrounds: water and land. The testing is OOD to the training because of the different mechanisms in combining the foreground and background.
\textbf{ImageNet-Rendition} \cite{Rendition} has renditions of 200 ImageNet classes, including art, cartoons, etc, which is an OOD test set for ImageNet. 
\textbf{ImageNet-Sketch} \cite{wang2019learning} contains sketch of 1000 ImageNet classes, which evaluate classifiers' robustness without texture and color clue. 
\textbf{ImageNet-9 Backgrounds Challenge} \cite{advBG} studies the classifier's vulnerability to adversarially chosen backgrounds on ImageNet.


\subsection{Baselines}

Our paper studies generalization on the out-of-distribution test set without domain index for training samples. We compare with the following baselines:

\emph{ERM} \cite{vapnik1992principles, domainGeneralization} is the standard way to train deep network classifiers.
\emph{GenInt} \cite{mao2021generative} learns a causal classifier by steering the generative models to simulate interventions. 
\emph{RSC} \cite{huangRSC2020} uses representation self-challenging to improve generation to the OOD data, where features that are significant in ERM will be punished.
We also compare with the popular IRM \cite{arjovsky2020invariant} which uses domain index information.

\subsection{Experimental Settings}
We construct the low capacity network $\hat{P}(Y|X', R)$ with 3 random convolution layers applied to a bag of patches from $X'$, concatenating the obtained feature with $R$, and then using 2-layer fully connected network to predict $Y$. Except for CMNIST where the input is low dimension and we do not use convolution layer. We set $N_j=256$ and $N_i=10$ for all experiments and denotes it as \textbf{Ours}. We also conduct a variant with $N_j=1$ and $N_i=1$ and denote it as \textbf{Ablation}, where everything is the same as `Ours' but the inference procedure is a traditional single forward pass.
For CMNIST and WaterBird datasets,  we select the model with the highest validation accuracy. For ImageNet-Rendition and ImageNet-Sketch, we report the best validation accuracy as there is no validation/test split available.


\subsection{Results on Simulated Datasets}
\textbf{CMNIST.} Our approach uses the latent representation from VAE to construct the representation variable. We report the accuracy in Table \ref{tab:colorfulmnist}.  Our method outperforms existing methods including the causal GenInt method by over 20\%.

\textbf{WaterBird.}
Following prior work, we use the representation from a pre-trained ResNet50.
We train the model for 10 epochs. In Table \ref{tab:waterbird}, without using domain index information, our causal approach improves the out-of-distribution test performance by over 25\% compared with ERM, and even 1\% higher than the state-of-the-art GDRO \cite{GDRO} method which uses domain index information.

\textbf{ImageNet-9 Adversarial Backgrounds}. We assess our model's robustness on testing distributions where the foreground and the background are manipulated to be different from the training distribution. In Table \ref{tab:i9}, we experiment on three variants of contrastive loss based self-supervised learning approaches, including Moco-v2\cite{mocov2}, SWAV \cite{swav}, and SimCLR \cite{chen2020simple}. Overall, our approach performs better when the foreground object is present even if the background is changed.


\subsection{Real-world Out of Distribution  Generalization}






\textbf{ImageNet-Rendition} and \textbf{ImageNet-Sketch} are two OOD test sets for ImageNet. We study the representation from contrastive-loss-based self-supervision learning approaches including SimCLR, MoCo-v2, and SWAV. In addition, we also study the representations from supervised learning, though they may be imperfect representations. We show results in Table \ref{tab:ImageNetOOD}.  Our algorithm estimates the causal invariance, which improves OOD generalization. The exception is that the supervised trained models, ResNet50 and ResNet152, are not trained with contrastive learning and therefore may lose causal information.

\subsection{Analysis}

\textbf{Importance of Image Sampling.} Our approach requires marginalizing over random input images $x'$ at inference time. Sampling fewer $x'$ can speed up the inference, however, at a cost of not estimating the accurate causal effect. In Figure \ref{fig:ablation_nj}, we vary the number of samples $N_j$ and test the performance on four datasets. In general, We find for datasets with $K$ categories, using $N_j>K$ can significantly improve generalization.




\textbf{GradCam Visualization.} Using the criterion derived in the previous section, we expect our model to attend to the spatial regions corresponding to the object, instead of the spurious context. In Figure \ref{fig:gradcam}, we validate this by visualizing the regions that the models use for classification with the GradCAM \cite{selvaraju2017grad}. We examine four datasets, including the WaterBird, ImageNet-9, ImageNet-Rendition, and ImageNet-Sketch. We visualize the ERM model in the `Baseline' column, the branch that conditions on the variable $X$ of model $P(Y|R, X)$ in the `Spurious' Column, and our causal method in `Ours'. By discarding the information in the `Spurious' model through marginalizing over $X'$, our model focus on the right object for prediction.


\section{Conclusions}

Generalization is a fundamental problem in visual recognition. This paper uses causal transportability theory to revisit and formulate the problem of out-of-distribution classification, since associational relations are not generalizable across domains. Our results demonstrate improved out-of-distribution robustness on both simulated and real-world datasets. Our findings suggest integrating causal knowledge and tools into visual representations is a promising direction to improve generalization.

{\small
\textbf{Acknowledgments:} CM, JW, and CV are partially supported by DARPA SAIL-ON and DARPA GAIL. CM and JF are partially supported by DiDi Faculty Research Award, J.P. Morgan Faculty Research Award, Accenture Research Award, ONR N00014-17-1-2788, and NSF CNS-1564055. EB and KX are partially supported by NSF, ONR, Amazon, JP Morgan, and The Alfred P. Sloan Foundation. HW is partially supported by NSF Grant IIS-2127918 and an Amazon Faculty Research Award. 
}

{\small
\bibliographystyle{ieee_fullname}
\bibliography{egbib}
}

\end{document}